\documentclass{article}
\usepackage[preprint]{neurips_2026}

\usepackage[utf8]{inputenc}
\usepackage[T1]{fontenc}
\usepackage[hidelinks]{hyperref}
\usepackage{url}
\usepackage{microtype}
\usepackage{nicefrac}
\usepackage{graphicx}
\usepackage{amsmath}
\usepackage{amsfonts}
\usepackage{amsthm}
\usepackage{booktabs}
\usepackage{algorithm}
\usepackage{algpseudocode}

\newtheorem{theorem}{Theorem}

\newtheorem{remark}{Remark}

\newenvironment{proofsketch}{\begin{proof}[Proof sketch]}{\end{proof}}
\usepackage{array}
\usepackage{xcolor}
\usepackage{makecell}
\usepackage{xspace}
\usepackage{enumitem}

\usepackage{mathtools} 
\usepackage{amssymb}
\usepackage{appendix}
\def\CX{\mathcal{X}} 
\def\CA{\mathcal{A}}

\def\Zspace{\mathcal{Z}}

\def\tb{{b}}

\newcommand{\emathbb}[1]{\ensuremath{\underset{{#1}}{\mathbb E}}}

\newcommand{\probz}[1]{\ensuremath{\mathbb{P}_{Z}({#1})}}

\newtheorem{assumption}{Assumption}[section]

\newcommand{\prob}[1]{\ensuremath{\mathbb{P}({#1})}}

\def\voro{\mathcal{C}} 
\def\th{\tilde{h}} 
\def\tb{\tilde{b}}

\title{Finite-Time Analysis of MCTS in Continuous POMDP Planning}

\author{
\begin{tabular}[t]{@{}c@{\hspace{1.6em}}c@{}}
\begin{tabular}[t]{@{}c@{}}
\textbf{Da Kong} \\
\normalfont\small Technion Autonomous Systems Program (TASP) \\
\normalfont\small Technion -- Israel Institute of Technology
\end{tabular}
&
\begin{tabular}[t]{@{}c@{}}
\textbf{Vadim Indelman} \\
\normalfont\small Stephen B. Klein Faculty of Aerospace Engineering \\
\normalfont\small Faculty of Data and Decision Sciences \\
\normalfont\small Technion -- Israel Institute of Technology 
\end{tabular}
\end{tabular}
}

\begin{document}
\maketitle
\begin{abstract}
This paper presents a finite-time analysis for Monte Carlo Tree Search (MCTS) in Partially Observable Markov Decision Processes (POMDPs), with probabilistic concentration bounds in both discrete and continuous observation spaces. While MCTS-style solvers such as POMCP achieve empirical success in many applications, rigorous finite-time guarantees remain an open problem due to the non-stationarity and the interdependencies induced by heuristic action selection (e.g., UCB). In the discrete setting, we address these challenges by extending the polynomial exploration bonus to UCB in POMDP setting, yielding polynomial concentration bounds for the empirical value estimation at the root node.
For continuous observation spaces, we introduce an abstract partitioning framework and propose a finite-time bound on partitioning loss. Under mild conditions, we prove high-probability bound on value estimates in POMDPs with continuous observation space. Specifically, we propose Voro-POMCPOW, a variant of POMCPOW with finite-time guarantees that adaptively partitions the continuous observation space using Voronoi cells. This approach maintains a finite branching factor while preserving the original observation generator. Empirical validation demonstrates that the proposed Voro-POMCPOW shows competitive performance while providing theoretical guarantees. Although our analysis focuses on continuous POMDPs, the techniques developed herein are also applicable to continuous MDPs, closing another gap on the MDP side.
\end{abstract}

\section{Introduction}
Partially Observable Markov Decision Processes (POMDPs) provide a prime mathematical framework for sequential decision-making under uncertainty, which models all kinds of uncertainty into the observation model and transition model. While the exact solution of POMDPs is computationally intractable due to the curse of dimensionality and the curse of history, online anytime algorithms have emerged as efficient and powerful approximate methods that scale to large problems. Specifically, Monte Carlo Tree Search (MCTS) demonstrates strong empirical success in POMDP planning, notably POMCP~\citep{Silver10nips}, in domains ranging from game playing to robotics~\citep{Kurniawati22ar,Lauri22tro}.  

Despite their widespread practical success, the theoretical analysis of MCTS in POMDPs—particularly the finite-time performance guarantees—remains under-explored. 
To our knowledge, the most up-to-date existing theoretical analysis provides asymptotic convergence guarantees, like in POMCP~\citep{Silver10nips}, which requires the simulation count to grow without bound but comes with no finite-time guarantees. For many safety-critical applications, such as autonomous driving or robot manipulation, finite-time guarantees are essential.
A finite-time analysis of MCTS is challenging due to the non-stationary Multi-Armed Bandit (MAB), and interdependencies among samples. The recent finite-time analysis of MCTS~\citep{Shah22or} is limited to MDPs with a finite branching factor. 

In the context of POMDPs, MCTS-style solvers with finite-time guarantees are limited. As far as we know, the only existing work is by~\citet{Barenboim26aij}, which provides deterministic anytime bounds for discrete POMDPs. However, probabilistic finite-time guarantees for MCTS in POMDPs are still missing. Moreover, the deterministic bound from~\citep{Barenboim26aij} does not extend to continuous observation spaces, which are prevalent in real-world applications. A finite-time analysis for continuous spaces in both MDPs and POMDPs is an open problem.

This work bridges the aforementioned gaps by providing the first finite-time analysis of MCTS in POMDPs with probabilistic guarantees, in both discrete and continuous observation spaces.
We first extend the polynomial exploration bonus of UCB~\citep{Shah22or} to the discrete POMDP setting, deriving polynomial concentration bounds for the empirical value estimate at the root node.
Building on this, we propose the Corrected-POMCP solver, which incorporates the corrected UCB bonus and maintains finite-time guarantees.

Next, we extend the analysis to continuous observation spaces. We introduce an adaptive partitioning framework to discretize the continuous observation space, thereby inducing a finite branching factor. Under mild assumptions and requirements for the partitioning scheme, we derive a probabilistic bound for MCTS value estimates in continuous POMDPs by applying a union bound over the partition-induced error and the MCTS estimation error. As a specific instance, we employ  Voronoi partitioning and propose Voro-POMCPOW solver, an extension of POMCPOW with finite-time guarantees. Experimental results demonstrate  competitive performance of both of the proposed solvers.

Although our primary focus is on continuous POMDPs, the techniques developed are directly applicable to continuous MDPs, thereby also contributing to the finite-time analysis of MCTS in fully observable continuous domains.

Our contributions are summarized as follows:
\begin{itemize}
    \item We present a finite-time analysis for MCTS in discrete POMDP planning with probabilistic guarantees, which can be applied to a corrected variant of POMCP.

    \item We present the first finite-time analysis to  continuous observation and state spaces for POMDPs by introducing an abstract partitioning scheme. This methodology is also applicable to continuous MDPs. 

    \item We propose solvers with formal finite-time guarantees in discrete and continuous POMDPs, namely Corrected-POMCP and Voro-POMCPOW, and validate their performance.    
\end{itemize}

\section{Background}

POMDPs constitute a foundational framework for sequential decision-making under uncertainty. For a comprehensive background on POMDPs, we refer to surveys~\citep{Kurniawati22ar,Lauri22tro}. 

This paper focuses on the finite-time analysis of MCTS in the context of POMDPs. For a general overview of MCTS, we refer to the survey~\citep{Swiechowski23air}. In the fully observable MDP setting,~\citet{Shah22or} provide a finite-time analysis with probabilistic bounds for MCTS in finite-branching MDPs. \citet{Dam24jair} derive finite-time bounds for another variant of MCTS with a power-mean backup, later extended to continuous MDPs~\citep{Dam25icml2}.
In contrast, finite-time guarantees for MCTS in POMDPs remain largely open. To our knowledge, the only existing result is the deterministic anytime bound for discrete POMDPs by \citet{Barenboim26aij}. This work aims to bridge this gap by developing a finite-time analysis of MCTS for continuous POMDPs.

Beyond finite-time analysis, the convergence rate of MCTS is of particular interest for anytime solvers, as it directly helps set the computational budget. In MDPs, \citet{Chang25tac} provides a convergence rate analysis of MCTS with different heuristic functions under probabilistic bounds. This paper also aims to contribute to the convergence rate analysis of MCTS in POMDPs, with clear benefits for the design of POMDP simplification strategies (e.g.,~\citep{Kong24isrr}) in anytime planning.

A further key challenge addressed in this work is to handle  continuous spaces while preserving performance guarantees. Progressive widening, employed in solvers such as \textsc{POMCPOW}~\citep{Sunberg18icaps}, is a common technique for exploring in continuous spaces. Alternative geometric approaches include Voronoi partitioning: \citet{Lim21cdc} extend Optimistic Optimization~\citep{Munos11nips} to Voronoi Optimistic Optimization (VOO) for continuous action selection, and \citet{Hoerger24ijrr} construct a Voronoi-partitioned belief tree for continuous POMDPs. However, these methods do not come with formal guarantees. MCTS POMDP approaches that support continuous observation spaces while maintaining performance guarantees remain an open problem.

\section{Preliminary}
\subsection{Partially Observable Markov Decision Processes}
\label{subsec:pomdp-basics}
The basic model of Partially Observable Markov Decision Processes (POMDP) are defined as a tuple:
$\langle\CX, \CA, \Zspace, \mathbb{P}_{T}, \mathbb{P}_{Z}, r, \gamma, b_0\rangle$, where $b_0$ is the initial belief, $\CX$ is the state space, $\CA$ is the action space, $\Zspace$ is the observation space.
The transition model (or motion model) is defined as $\mathbb{P}_{T}(x_{k+1}|x_k, a_k)$, which describes the probabilistic transition of the state from $x_k\in\CX$ to $x_{k+1}\in\CX$ under a certain action $a_k \in \CA$. The observation model is defined as $\probz{z_k|x_k}$, which describes the probability of observation $z_k\in\Zspace$ given a certain state $x_k\in\CX$.
The reward function is considered to be state-dependent, $r: \mathcal{X},\mathcal{A} \mapsto \mathbb{R}$, and satisfies the bounded reward assumption: $r \in [-R_{\max}, R_{\max}]$. 

Given that the true state is uncertain, a belief is maintained to represent the distribution of the current state with regard to history.
The belief at time step $k$ is defined as $b_k \triangleq \prob{x_{k}| h_{k}} $, where the history $h_{k}$ is  defined as $h_{k} \triangleq \{z_{1:k},a_{0:k-1}\}$. 
We define the corresponding history space  as $\mathcal{H}_k \triangleq \{h_k\}$, which includes all possible history $h_k$.
A propagated history without the latest observation $h_k^{-}$ is defined as: $h_k^{-} = \{z_{1:k-1},a_{0:k-1}\}$, and the corresponding propagated belief is $b^-_k \triangleq \prob{x_k|h^-_k}$.
The Bayesian belief update is given by:
$    b_k = P(x_k|h_k) = \eta_k \; P(z_k \mid x_k)  P(x_k \mid h_k^-) = \eta_k \; P(z_k \mid x_k) b_k^-,
$ where $\eta_k$ is the Bayesian factor. 
We define the Bayesian belief update as the operator  $b_{k+1}=\psi(b_k, a_k, z_k) $.


The value function for a deterministic policy $\pi$ over the planning horizon $L$ is defined as: 
\(   V^\pi(b_k)= r(b_k, \pi_k(b_k)) + \sum_{i=k+1}^{k+L}\gamma^{i-k}\emathbb{z_{k+1:i}|b_k,\pi}\big[ r(b_i,\pi_i(b_i))\big].
\)
Here, for state-dependent rewards,  $r(b
, a) = \emathbb{x\sim b} [r(x, a)]$, and $\pi=(\pi_0, \ldots, \pi_L)$ with $\pi_i: \mathcal{H}_i \mapsto \mathcal{A}$, where $\mathcal{H}_i$ is the history space at time step $i$. In this paper, we use the history and belief space interchangeably. 
The corresponding $Q$-function is $Q^\pi(b_k, a_k) = r(b_k, a_k) + \emathbb{z_{k+1}|b_k,a_k} \big[ V^{\pi}(b_{k+1}) \big]$.

The goal of a POMDP is to find the optimal policy sequence $\pi^*=(\pi^*_0,\ldots\pi^*_{L})$ that maximizes the value function: 
\(
    \pi^* = \operatorname*{argmax}_{\pi} V^{\pi}(b_k).
\)
The optimal value function satisfies the Bellman optimality equation:
\(    
V^{*}(b_k) = \max_{a_k} \Big[ r(b_k, a_k) + \emathbb{z_{k+1}|b_k,a_k} \big[ V^{*}(b_{k+1}) \big] \Big].
\)
While the beliefs in the formulation above are theoretical, in practice these are often approximately represented by particle beliefs. Unless otherwise mentioned, we also consider this setting throughout the paper. Extending our results to the level of theoretical beliefs is possible, following a similar  methodology as in \citep{Lim23jair}, but outside the scope of this work.


\subsection{MCTS in POMDP Planning}

Monte Carlo Tree Search (MCTS) is a sampling-based, anytime planning algorithm that incrementally constructs a search tree under a given computational budget. When time budget permits, the algorithm performs successive iterations, each consisting of four stages: (i) \emph{selection}, which traverses the tree from the root by recursively choosing child nodes according to a tree policy that balances exploration and exploitation, typically the Upper Confidence Bound (UCB) rule; (ii) \emph{expansion}, which adds one or more unvisited nodes upon reaching a leaf; (iii) \emph{simulation}, which generates a trajectory from the newly expanded node to a finite planning horizon using the default or rollout policy; and (iv) \emph{backup}, which propagates the simulated return backward along the visited trajectory to update value estimates and visitation counts.

In an MDP setting, MCTS operates directly on state--action nodes, with the root corresponding to the initial state. In a POMDP setting, since the true state is unknown, planning must be performed over \textit{beliefs} derived from action--observation histories. A direct adaptation of MCTS to POMDPs is to represent the tree by belief nodes, each associated with a history, and use the given generative model of the environment to sample  trajectories, thereby constructing a belief tree.

A notable MCTS-based POMDP solver is POMCP~\citep{Silver10nips}, which employs a UCT-style tree search. Each simulation begins by sampling a particle from the initial belief, then selects actions according to the UCT criterion and samples subsequent states and observations from the generative model. Rather than maintaining an explicit belief representation, POMCP approximates the belief at each history node using a particle set derived from the simulated trajectories. The tree is indexed by action--observation histories, with each node storing visitation counts and estimated values. During backup, these statistics are updated using the simulated return. 

MCTS incrementally constructs an estimator of a stochastic policy sequence $\pi$, where for each time step $i$, $\pi_i: \mathcal{H}_i\times\mathcal{A}_i\mapsto [0,1]$ for a discrete setting. The estimated stochastic policy in MCTS is given by  $\hat{\pi}(a|b)=\frac{N(b,a)}{N(b)}$, where $N(b)$ is the visitation count of node $b$, and $N(b,a)$ is the visitation count of node $(b,a)$.

Applying MCTS-style solvers in continuous POMDPs shows a significant challenge, primarily due to the infinite branching factor in action and observation spaces. Several notable solvers have been developed to work in continuous spaces, including \textsc{POMCPOW}~\citep{Sunberg18icaps}, \textsc{VOMCPOW}~\citep{Lim21cdc}, and \textsc{ADVT}~\citep{Hoerger24ijrr}. These methods typically adopt a progressive widening strategy to selectively expand nodes in continuous spaces, thereby rendering the search tractable.
However, existing methods lack dedicated mechanisms that can provide performance guarantees in continuous observation spaces.

In MCTS, the value estimation is derived from the average of the sampled returns. 
Let $G^i(b_k)$ denote the return sample obtained from the $i$th iteration at the belief node $b_k$. 
The empirical value estimate after $N$ simulations at $b_k$ is given by:
$
    \hat{V}_N(b_k) = \frac{1}{N} \sum_{i=1}^N G^i(b_k).
$


\section{Finite-time Analysis in Discrete POMDP}
\label{section:discrete}
This section serves as a warm-up examination of the finite-time analysis of MCTS-style solvers in discrete POMDPs, with a particular focus on POMCP~\citep{Silver10nips}. 

We first identify two principal challenges in the theoretical analysis: the \textbf{non-stationary} nature and the \textbf{interdependencies} inherent in the MCTS tree structure. The non-stationarity arises from the anytime simulation process: as iterations progress, the stochastic policy used for the value estimation shifts continuously. The interdependencies stem from heuristic-based action selection mechanisms, such as UCB1, which depend on historical selections and accumulated rewards. This dependency violates the i.i.d.\ assumption required by Hoeffding's inequality, thereby complicating the analysis.

For UCB-based methods specifically, an additional obstacle to establishing finite-time bounds---particularly probabilistic bounds---is the logarithmic term in the UCB function employed during action selection. As noted in~\citep{Audibert09colt}, UCB cannot achieve exponentially decaying bounds; only polynomial bounds are attainable. By extending the corrected UCB formulation proposed in~\citep{Shah22or} to the POMDP setting, this section presents a finite-time analysis of MCTS for discrete POMDP planning, establishing a probabilistic concentration bound in discrete settings.

\subsection{Corrected-POMCP}
We propose to replace the standard logarithmic exploration bonus of UCB1 in POMCP by a \emph{polynomial} bonus form~\citep{Shah22or}. Concretely, for any tree node with history $h$ and action $a$, the action selection process is conducted by:
\begin{align}
\label{eq:corrected-index}
a_t \in \arg\max_{a\in\mathcal A(h)}
\left\{
\hat Q(h,a) + B_{t,s}(h,a)
\right\},
\end{align}
where $\hat Q(h,a)$ is the empirical return estimate and the polynomial exploration bonus is
\begin{align}
\label{eq:corrected-bonus}
B_{t,s}(h,a)
\triangleq
\beta^{1/\xi}\; N(h)^{\alpha/\xi}\; N(h,a)^{-(1-\eta)}.
\end{align}
Here $N(\cdot)$ is the visitation count, $\beta>1$, $\xi>0$ and $\eta\in[1/2,1)$ are concentration parameters, and
$\alpha$ is a tuning exponent.

By plugging the corrected UCB bonus into POMCP, we obtain the Corrected-POMCP solver. The solver largely follows POMCP's process except for the action selection process. 
The detailed algorithm is presented in the Appendix as Algorithm~\ref{alg:corrected-pomcp}.

\newcommand{\EE}{\ensuremath{\mathbb{E}}}
\newcommand{\PP}{\ensuremath{\mathbb{P}}}
\subsection{Theoretical Analysis}
This part of the theoretical analysis for discrete POMDPs largely follows the structure established for finite-branching-factor MDPs by \citet{Shah22or} (with their extension to stochastic models in the Appendix), extending it to a discrete POMDP setting.


We begin by stating the assumptions underlying our analysis of discrete POMDPs.

\begin{assumption}
The reward function $r(b, a)$ is bounded such that $|r(b,a)| \leq R_{\max}$.
\label{assumption:bounded-reward}
\end{assumption}
\begin{assumption}[Observation support and branching factor]
\label{assumption:observation-factor}
There exists $\phi_Z > 0$ such that for every reachable tree node (history or belief) $x$ and action $a$:
\begin{itemize}[leftmargin=*]
    \item[i.] The observation distribution $\PP_Z(\cdot \mid x, a)$ has finite support;
    \item[ii.] For any observation with positive probability,
    \begin{align}
        \PP_Z(z \mid x, a) \geq \phi_Z, \ \forall z \text{ such that } \PP_Z(z \mid x, a) > 0.
    \end{align}
\end{itemize}
Consequently, each state-action pair $(x, a)$ admits at most
$M_Z \triangleq \bigl\lfloor \phi_Z^{-1} \bigr\rfloor$
observation children in the belief tree.
\end{assumption}

Building upon these assumptions, we now establish the notation and framework for the finite-time analysis.

For each action $a \in \CA$, let $\{G_{a,t}\}_{t \geq 1}$ denote the sequence of return samples generated by MCTS when the simulation at $b_k$ initiates with action $a$ at the first decision step and subsequently follows the simulation/rollout policy.

For $n\ge1$, we define the following quantities:
$ \bar G_{a,n} \triangleq \frac{1}{n}\sum_{t=1}^n G_{a,t}$, $
  \mu_{a,n} \triangleq \EE[\bar G_{a,n}]$, and $ 
  Q_a \triangleq\lim_{n\to\infty} \mu_{a,n}.$ 
Furthermore, we define: 
$  V^*(b) \triangleq \max_{a\in\CA} Q(b,a), 
  a^* = \arg\max_{a\in\CA} Q(b,a)
$, and 
$  \Delta_{\min} \triangleq \min_{a\neq a^*} (V^* - Q_a),
  \delta_{*,n} \triangleq \mu_{a^*,n} - V^*.
$

At the $t$-th simulation from $b_k$, let $A^t \in \CA$ denote the action selection at $b_k$, and let
$N_a(t) \triangleq \sum_{\ell=1}^{t} \mathbb{I}\{A^\ell = a\}$
represent the cumulative number of times action $a$ has been selected up to simulation $t$. We assume that the action selection process adheres to the corrected-UCB rule defined in~\eqref{eq:corrected-index}.

\begin{theorem}[Polynomial Concentration of MCTS in Discrete POMDPs]
\label{theorem:pomdp-mcts-ucb}
Let Assumptions~\ref{assumption:bounded-reward} and~\ref{assumption:observation-factor} hold. Consider a POMDP belief tree of depth $L$, where the root $b_0$ is at depth $h=0$ and leaf nodes are at depth $h=L$.
Let $\eta^{(h)} \in [1/2, 1)$ be a fixed constant. Define the sequences of parameters $\{\xi^{(h)}\}_{h=0}^{L}$ and $\{\beta^{(h)}\}_{h=1}^{L}$ recursively updated from the leaves to the root (as shown in Appendix~\ref{appendix:proof-pomdp-mcts-ucb}.) Suppose the MCTS solver selects actions at any node $h$ (at depth $h < L$) using the Corrected-UCB index with level-dependent parameters $(\beta^{(h+1)}, \xi^{(h+1)},\eta^{(h+1)})$ from the child level.

Then, there exist constants $\beta^{(0)} > 1$ (depending on $V_{\max}, |\mathcal{A}|, L, \Delta_{\min}, M_Z, \phi_Z$) such that for the root value estimate $\hat{V}_n(b_0)$ after $n$ simulations, the following concentration bounds hold for all $n \ge 1$ and $z \ge 1$:
\begin{align}
    \mathbb{P}\left( n \left( \hat{V}_n(b_0) - V^*(b_0) \right) \ge n^{\eta^{(0)}} z \right) &\le \beta^{(0)} z^{-\xi^{(0)}}, \\
    \mathbb{P}\left( n \left( \hat{V}_n(b_0) - V^*(b_0) \right) \le -n^{\eta^{(0)}} z \right) &\le \beta^{(0)} z^{-\xi^{(0)}}.
\end{align}

\end{theorem}
\begin{proofsketch}
The proof largely follows \citet{Shah22or} with their extension to stochastic models in the Appendix (Lemma A.1). Here we replace the state mixture process in MDP with observation mixture in POMDP. The proof is in Appendix~\ref{appendix:proof-pomdp-mcts-ucb}, with a closed-form parameter design in Appendix~\ref{app:parameter-design}.
\end{proofsketch}

\vspace{-0.3cm}

Theorem~\ref{theorem:pomdp-mcts-ucb} provides a finite-time analysis of the empirical value estimate $\hat V_n(b_k)$ of MCTS-style solvers in a discrete setting, specifically for the Corrected-POMCP. The theorem establishes both convergence and polynomial concentration bounds for the empirical value estimate. The polynomial concentration result provides high-probability bounds on the deviation of the empirical value estimate from the true value, with the deviation shrinking polynomially in the number of samples $n$.
The bounds are given in terms of the parameters $\beta^{(0)}$, $\xi^{(0)}$, and $\eta^{(0)}$, which depend on the problem parameters and are updated iteratively from the leaf nodes to the root node.


\section{Finite-time Analysis in Continuous POMDP}
This section extends the finite-time analysis to the continuous setting. We first identify a notable gap in the existing literature concerning theoretical analysis for POMDP planning in continuous spaces. The finite-time analyses for POMDPs include the probabilistic concentration bounds described in Section~\ref{section:discrete}, which require a finite branching factor in the belief tree, as well as the deterministic bounds provided in~\citep{Barenboim26aij}, which are similarly restricted to discrete POMDPs.

In this section, we introduce an adaptive partitioning scheme for continuous spaces in POMDP planning and establish a theoretical foundation for  a  finite-time analysis in continuous POMDPs. Specifically, we focus on the continuous observation space, as the continuous action space can be addressed through a similar progressive widening approach, as demonstrated in VOMCPOW~\citep{Lim21cdc}; the extension is straightforward.

\subsection{Abstract Partitioning of Continuous Spaces}
\label{subsec:abstract-partition}
To address the infinite branching factor arising from continuous observation spaces, we propose an abstract partitioning scheme to discretize the continuous space.

We define a partition for each belief node with history $h,a$ as $\mathcal{P}(h,a)$, where each belief node has its own partition and each partition comprises a collection of cells $\mathcal{C}(h,a) = \{c^i(h,a)\} \in \mathcal{P}(h, a)$. We denote the number of cells as $m(h,a)\triangleq\big|\voro(h, a)\big|$. For simplicity, we will omit the history for each cell in this section.

For each cell $c$, we denote the set of observations associated with the cell as $\mathcal{Z}(c) = \{z \in c\}\subseteq \mathcal{Z}$ and designate a representative center $z^c(c)$. For simplicity, we denote the center as $z^c$ in the later part.
The boundaries between cells may be either explicit or implicit, depending on the specific partitioning method employed. Specifically, in conjunction with the partitioning, we define a general metric $D(\cdot, \cdot)$ to measure the distance between two observation realizations in the corresponding observation space. Typically, observations are assigned to the nearest cell center according to this metric, i.e.~$\mathcal{Z}(c|\mathcal P(h,a)) = \bigl\{
	z \;\big|\;
	D(z, z^c) \le D(z, z^{c'}),
	\ \forall\, z^{c'} \neq z^c
	\bigr\}$.


%
Based on this partitioning, we modify the history representation in POMDPs to incorporate cell information in place of raw observations:
\(    \th_t = (b_0, a_0, c_1, a_1, c_2, \cdots, a_{t-1}, c_t).
\) 
However, for partitionings with explicitly defined boundaries, the cell information is difficult to represent directly. We propose an alternative, commonly used approach that utilizes only the center point $z^c$ to represent all observations within a given cell. Under this formulation, the history becomes:
\(    \th_t = (b_0, a_0, z^c_1, a_1, z^c_2, \cdots, a_{t-1}, z^c_t).
\) 

Under this center point representation, the value function is defined as:
\(
		V^{p}(\tb_k)
		= r(b_k, \pi^{p}(\tb_k))  +
		\sum_{c_{k+1}} P(c_{k+1} | \tb_k, \pi^p(\tb_k)) V^{p}(\tb_{k+1}),
\)
where  $\pi^p$ is a policy with policy space $\Pi^{p}: \tilde{\mathcal{H}}_t \rightarrow \mathcal{A}$, where $\tilde{\mathcal{H}}_t$ is the set of all possible histories $\th_t$.
The belief update is given by:
\(    \tb_{k+1}\triangleq P(x_{k+1}\mid \th_{k+1}) = \psi(\tb_k, \pi^p(\tb_k), z^c_{k+1}).
\)

\begin{remark}[Justification for Center Point Representation]
    The center point representation for each cell is commonly used in many MCTS-style solvers that partition the action space, e.g.~\citep{Lim21cdc,Hoerger24ijrr}. The intuition is that if the cell is sufficiently small, the center point serves as a good representation of the cell. Although the center point representation sacrifices some fidelity compared to the full observation representation, it will facilitate the MCTS algorithm by visiting the same belief multiple times and improve the estimation of the value function at that belief.
\end{remark}
\vspace{-0.1cm}

\subsection{Theoretical Analysis}



\begin{figure}[t]
    \centering
    \includegraphics[width=0.5\columnwidth]{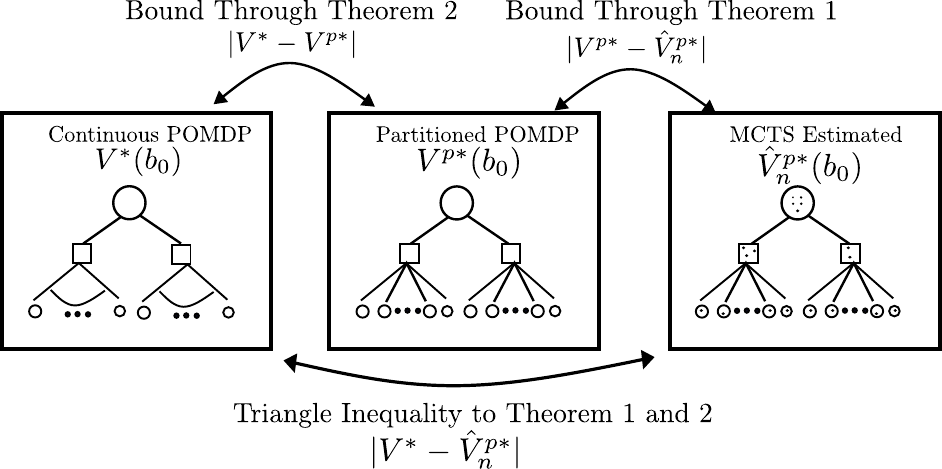}
    \caption{The flow of the theoretical analysis. We bound the continuous POMDP with the partitioned POMDP at the theoretical value function level. Then we bound the theoretical and MCTS-estimated value function for the partitioned POMDP. Finally, we use the triangle inequality to combine the two bounds to obtain a final bound for the MCTS estimates in continuous POMDPs. 
    }
    \label{fig:bound-flow}
\end{figure}


In this section, we provide a probabilistic bound for the MCTS estimation error in continuous POMDPs:
$| V^*(b_0) - \hat{V}^{p*}_n(b_0) |.$
Here, $V^*(b_0)$ is the theoretical value function for continuous POMDP and $\hat{V}^{p}_n(b_0)$ is the MCTS-estimated value function after $n$ iterations for the partitioned POMDP.

Our analysis proceeds in two steps: First, we establish a bound between the theoretical value function for the original problem and for the partitioned POMDP, $|V^*(b_0) - V^{p*}(b_0)|$. Then, as we have already formed  a finite-branching-factor POMDP via the abstract partitioning, we directly apply Theorem~\ref{theorem:pomdp-mcts-ucb} to obtain a bound  between the MCTS-estimated and the theoretical value function for the partitioned POMDP: $|V^{p*}(b_0) - \hat{V}^{p*}_n(b_0)|$.
Finally, we use the triangle inequality to combine the two bounds to obtain a final bound for the MCTS estimates in continuous POMDPs,
\begin{align}
    \label{eq:bound-decompose}
|V^*(b_0) - \hat{V}^{p*}_n(b_0)|
&\le |V^*(b_0)-V^{p*}(b_0)|+|V^{p*}(b_0)-\hat{V}^{p*}_n(b_0)|.
\end{align}
The overall flow of the theoretical analysis is shown in Fig.~\ref{fig:bound-flow}.

Before presenting the bounding theorem between the theoretical value function for the original problem and for the partitioned POMDP, we first state the assumptions and requirements for the POMDP setting.
\begin{assumption}[H\"older continuity of belief update in observation]\label{ass:holder-obs}
There exist constants $L_\psi>0$ and $\alpha\in(0,1]$ such that for all $b\in\mathcal B$, $a\in\mathcal A$ and $z,z'\in\mathcal Z$, \(W^1\big(\psi(b,a,z),\,\psi(b,a,z')\big)\le L_\psi D(z,z')^{\alpha}\), where $W^1$ is the Wasserstein-$1$ distance, and $\psi$ is the belief update function.
\end{assumption}

\begin{assumption}[H\"older continuity of the optimal value in belief]\label{ass:holder-V-function}
There exist constants $L_V>0$ and $\beta\in(0,1]$ such that for all $b,b'\in\mathcal B$, \(\big|V^*(b)-V^*(b')\big|\le L_V W^1(b,b')^{\beta}\).
\end{assumption}
%
%
\noindent We define the node-wise covering radius and cell diameter by \(\varepsilon_{\th a}:=\sup_{z\in\mathcal Z}\min_{c\in\voro(\cdot\mid \th a)} D\big(z, z^{c}(\th a)\big)\) and \(\operatorname{diam}\!\big(\voro(\cdot\mid \th a)\big):=\max_{c\in\voro(\cdot\mid \th a)}\sup_{z,z'\in c} D(z,z')\).

\begin{assumption}\label{ass:covering_dvp}
There exist constants $C_{\mathrm{cov}}>0$ and $k_{Z}>0$ such that for all $r\in(0,1]$, \(N(Z,r)\le C_{\mathrm{cov}} r^{-k_{Z}}\) and \(\varepsilon_{\th a}\le 2r\),
where $N(Z,r)$ is the smallest number of $k_Z$-balls of radius $r$ needed to cover $Z$.
\end{assumption}

\begin{assumption}\label{ass:ball-mass}
Fix any history--action node $\th a$ that can appear in the search tree. Let $\nu_{\th a}$ denote the (marginal) distribution on
$\mathcal{Z}$ of observations generated at node $\th a$, i.e., $P(z|\th,a)$.
There exist constants $c_\nu>0$ and $r_0>0$ (uniform over all such $\th a$) such that for all $z\in \mathcal Z$ and all $r\in(0,r_0]$, \(\nu_{\th a}\big(B_Z(z,r)\big)\ge c_\nu r^{k_{Z}}\), where \(B_Z(z,r)\triangleq \{z'\in \mathcal Z \!:\  \!\! D(z,z')\le r\}\).
\end{assumption}

\begin{remark}[Justification for Assumptions]
The assumptions adopted in this work are mild and commonly used in the literature. The H\"older continuity assumption is a relaxed version for Lipschitz continuity, which is used in analyzing continuous POMDP planning~\citep{Kim20aaai,Lim21cdc}. Our polynomial shrinkage ratio for covering numbers similarly constitutes a relaxed version of the exponential shrinkage ratio considered in~\citep{Kim20aaai}. 
\end{remark}

Then, we get the bounding theorem for partitioning loss.

\begin{theorem}\label{theorem:partition-loss}
Assume Assumptions~\ref{ass:holder-obs}--\ref{ass:ball-mass}. Fix a planning horizon $L$ and let
$\mathcal{\tilde{H}}_L$ be the set of reachable history--action nodes up to depth $L$.
At each $\th a\in\mathcal{ \tilde {H}}_L$, replace the raw observation $z$ by the nearest center
$\hat z=q_{\th a}(z)$ from $m(\th a)$ centers, with covering radius
$\varepsilon_{\th a}:=\sup_{z\in\mathcal Z}\min_j D\big(z,z^{c,j}(\th a)\big)$.
Then for any $\delta\in(0,1)$, with probability at least $1-\delta$,
\begin{align}
 |V^*(b_0)-V^{p*}(b_0) |
\le
\frac{\gamma(1-\gamma^{L})}{1-\gamma}\;
L_V L_\psi^\beta\;
\max_{\th a\in\mathcal{ \tilde{H}}_L}
\left[
C\left(\frac{\log\!\big(\frac{m(\th a)|\mathcal{ \tilde{H}}_L|}{\delta}\big)}{m(\th a)}\right)^{\frac{1}{k_{Z}}}
\right]^{\alpha\beta},
\label{eq:center-loss-concise2}
\end{align}
where $C>0$ is a constant depending only on $(C_{\mathrm{cov}},c_\nu,r_0, k_{Z})$, and $|\tilde{\mathcal{H}}_L|$ is the size of history nodes at depth $L$.
\end{theorem}
\begin{proofsketch}
    This can be proved by first bounding the error at each node directly using the assumptions, and then applying a union bound over all nodes. The proof is provided in Appendix~\ref{app:proof_theorem_2}.
\end{proofsketch}

We now present our \textit{main theorem} that bounds  the MCTS estimation error in continuous POMDPs. The following Theorem provides finite-time guarantees for MCTS in continuous POMDPs if the partitioning follows the above assumptions. 

\begin{theorem}[MCTS Estimation Bound in Continuous POMDPs]\label{theorem:overall-error}
Assume Assumptions~\ref{ass:holder-obs}--\ref{ass:ball-mass} and Assumption~\ref{assumption:bounded-reward}--~\ref{assumption:observation-factor} hold. Assume moreover that Theorem~\ref{theorem:pomdp-mcts-ucb} applies to the \emph{partitioned} (finite-branching) POMDP induced by the adaptive abstract partition, yielding concentration parameters $(\beta^{(0)},\xi^{(0)},\eta^{(0)})$ for the corresponding corrected-UCB MCTS estimator. Let $\hat V_n^{\,p}(b_0)$ be the root value estimate after $n$ simulations on the partitioned tree, and let $V_p^*(b_0)$ be the optimal value for the partitioned problem. Then for any $\delta_1,\delta_2\in(0,1)$, with probability at least $1-(\delta_1+\delta_2)$,
\begin{align}
\bigl|\hat V_n^{\,p}(b_0)-V^*(b_0)\bigr|
\le&
n^{-(1-\eta^{(0)})}\Bigl(\frac{2\beta^{(0)}}{\delta_1}\Bigr)^{1/\xi^{(0)}}
\nonumber \\ +&  
\frac{\gamma(1-\gamma^L)}{1-\gamma}\,L_V L_\psi^\beta\,
\max_{\tilde h a\in H_L}
\left[
C\left(
\frac{\log\!\bigl(\frac{m(\tilde h a)\,|H_L|}{\delta_2}\bigr)}{m(\tilde h a)}
\right)^{1/k_Z}
\right]^{\alpha\beta}.
\end{align}
In particular, taking $\delta_1=\delta_2=\delta/2$ yields a bound holding with probability at least $1-\delta$.
\end{theorem}

\begin{proofsketch}
The result follows by applying a union bound to combine the partitioning loss bound in Theorem~\ref{theorem:partition-loss} with the MCTS estimation bound for the discrete setting in Theorem~\ref{theorem:pomdp-mcts-ucb}. The observation noise from nonstationary partitioning can be addressed by a martingale analysis as discussed in the detailed proof in Appendix~\ref{proof:overall-error}.
\end{proofsketch}




We can employ progressive widening in MCTS to realize the adaptive partitioning,  the number of cells $m(\tilde{h}a)$ at node $\tilde{h}a$ is controlled so that
$m(\tilde{h}a) \;\le\; k_z\, N(ha)^{\alpha_z}.$
The analysis of the partitioning in Theorem~\ref{theorem:partition-loss} is general and can be applied to any partitioning method that satisfying the stated assumptions, and Theorem~\ref{theorem:pomdp-mcts-ucb} requires a fixed finite branching factor. Consequently, for the purpose of theoretical analysis, we consider a ``frozen'' instance of partitioning from the progressive widening process at a finite time.
 It is worth noting that in the adaptive setting, the observation mixture step in the proof of Theorem~\ref{theorem:pomdp-mcts-ucb} requires minor modifications but yields equivalent results, as detailed in proof of Theorem~\ref{theorem:overall-error} in Appendix~\ref{proof:overall-error}. A practical solver utilizing this framework is presented in the subsequent section.
A parameter design with closed-form representation is given in Appendix~\ref{app:parameter-design}.


\subsection{Voro-POMCPOW}

Voronoi partitioning is a practical method to partition the continuous space, where the border of each partition is defined by the distributions of the sampled control points in a non-stationary way.
In each Voronoi cell, all the points are closer to its center point than any other center points.
In POMDP planning, \citet{Lim21cdc} use Voronoi partitioning to handle a continuous action space, adopting Voronoi Optimistic Optimization (VOO) to balance exploration and exploitation in the continuous action space, and \citet{Kim20aaai} use it for deterministic MDPs.
However, it remains an open problem how to use Voronoi partitioning to discretize the state and observation spaces in POMDPs.

We propose to use Voronoi partition as the abstract partitioning introduced in the previous section to discretize the continuous observation space. For Voronoi partition, it is easy to satisfy the requirements of Assumptions~\ref{ass:holder-obs}--\ref{ass:ball-mass}.
 The partitioning is based on the distribution of the center points (or control points) $z^c$. Each cell contains the observations that are closer to its center point than any other center points. As the sampling process proceeds, the partitioning becomes non-stationary but finer. Since we use the center points to represent cells, the samples will not be re-assigned to new cells when created.

The proposed method, namely Voro-POMCPOW, is shown in Algorithm~\ref{alg:voro-pomcpow}. It extends the progressive widening mechanism in observation space of POMCPOW to a Voronoi-based Progressive Widening, which can lead to the finite-time guarantee. For this work, we only consider the Voronoi partitioning in the observation space and consider the discrete action space. It is trivial to extend it to the continuous action space by using Voronoi Optimistic Optimization (VOO) method as in~\citep{Lim21cdc}.

In this work, we consider a simple version of Voronoi partitioning, where each cell is represented by the center point only and the sample points will be used only for choosing cells and then be discarded. This follows the definition of in-cell estimation represented by the cell points only in Section~\ref{subsec:abstract-partition}.

\section{Empirical Results}
This section evaluates the empirical performance of the proposed solvers in comparison to POMCPOW. Then we present the practical implementation of the probabilistic finite-time guarantee for MCTS in continuous POMDPs. We present a limited empirical evaluation here due to the page limit, and the contribution of this paper is primarily theoretical.


\begin{table}[t]
\centering
\hfill
\begin{minipage}[t]{0.4\textwidth}
\centering
\caption{Mean returns for 10-step simulations over 100 episodes on Modified LightDark 1D.}
\label{tab:pomcpow-mean-results}
\small
\setlength{\tabcolsep}{3pt}
\begin{tabular}{l r}
\toprule
Method & Mean Return $\pm$ std  \\
\midrule
Voro-POMCPOW & 6.036653$\pm$ 0.273107  \\
POMCPOW   & 6.094985$\pm$0.300759    \\
\bottomrule
\end{tabular}
\end{minipage}
\hfill
\begin{minipage}[t]{0.4\textwidth}
\centering
\caption{Average post-run examination results for proposed bounds in Thm.~\ref{theorem:overall-error}. }
\label{tab:lightdark-post-run-bound-5000}
\small
\setlength{\tabcolsep}{3pt}
\begin{tabular}{c c c c}
\toprule
Confidence  & Proposed Bound in Thm.~\ref{theorem:overall-error} \\
\midrule
$80\%$  & 1.992210 \\
$85\%$  & 2.003508 \\
$90\%$  & 2.022460 \\
\bottomrule
\end{tabular}
\end{minipage}
\hfill
\vspace{-0.3cm}
\end{table}

We evaluate our method on the LightDark 1D benchmark, which has a discrete action space and a continuous observation space. We use POMCPOW as the baseline and compare it with the proposed Voro-POMCPOW. Table~\ref{tab:pomcpow-mean-results} reports the mean returns with standard deviations over 100 episodes. Although our method uses a more conservative polynomial exploration bonus to enable finite-time guarantees, it achieves competitive empirical performance compared with POMCPOW.
Table~\ref{tab:lightdark-post-run-bound-5000} reports the post-run average results of our proposed bound in practice under different confidence levels in the modified LightDark 1D environment with 5000 iterations.
The results demonstrate the practicality of the proposed bound in planning problems. The detailed experimental setup is provided in Appendix~\ref{appendix:experiment}.

\section{Conclusion}
This work presents the first finite-time analysis of MCTS in POMDPs, addressing the challenge of non-stationary value estimates. We established polynomial concentration bounds for the discrete setting and introduced an abstract partitioning framework for continuous observation spaces, shown by the Voro-POMCPOW. We also present empirical results for our proposed methods.
These contributions not only bridge the gap between theory and practice in POMDP planning but also offer new analytical tools applicable to continuous MDPs.

\paragraph{Limitations}
Our analysis assumes a fixed horizon planning problem and relies on assumptions on the observation space. Extending these results to infinite horizons remains an open challenge. Besides, this analysis can be extended to PFT-DPW-style algorithms, which is left for future work. We can also test the proposed bounds on more complex environments and solvers, which is left for future work as well.


\section{Acknowledgments}
This work was supported by the Israel Ministry of Innovation, Science and Technology.

\bibliographystyle{plainnat}
\bibliography{refs}

\clearpage
\appendix
\section*{Supplementary Materials}

\section{Algorithms and Parameter Design}
\subsection{Algorithm: Corrected-POMCP}
The algorithm for Corrected-POMCP, with the corrected UCB for the action selection process, is shown below:

\begin{algorithm}
\caption{Corrected-POMCP}
\label{alg:corrected-pomcp}
\begin{algorithmic}[1]
\Require History $h$, Generative Model $\mathcal{G}$, Parameters $\beta, \xi, \alpha, \eta$, Planning horizon $L$
\Function{Search}{$h$}
    \While{Budget $> 0$}
        \State $s \sim \mathcal{B}(h)$ \Comment{Sample state from belief}
        \State \Call{Simulate}{$s, h, 0$}
    \EndWhile
    \State \Return $\arg\max_{a \in \mathcal{A}} Q(h, a)$
\EndFunction
\Statex
\Function{Simulate}{$s, h, d$}
    \If{$d \ge L$} \Return $0$ \EndIf
    \If{$h \notin Tree$} \Comment{Expansion}
        \State Initialize $N(h)=0, \forall a: N(h,a)=0, Q(h,a)=0$
        \State \Return \Call{Rollout}{$s, h, d$}
    \EndIf

    \State \textbf{$a^* \leftarrow \arg\max_{a} \left( Q(h, a) + \frac{\beta^{1/\xi} \cdot N(h)^{\alpha/\xi}}{N(h, a)^{1-\eta}} \right)$} \Comment{\textbf{Polynomial Exploration Bonus} }

    \State $(s', z, r) \sim \mathcal{G}(s, a^*)$ \Comment{Call Generative Model}
    \State $R \leftarrow r + \gamma \cdot \Call{Simulate}{s', (h, a^*, z), d+1}$

    \State $N(h) \leftarrow N(h) + 1; \quad N(h, a^*) \leftarrow N(h, a^*) + 1$ \Comment{Update Statistics}
    \State $Q(h, a^*) \leftarrow Q(h, a^*) + \frac{R - Q(h, a^*)}{N(h, a^*)}$
\EndFunction
\end{algorithmic}
\end{algorithm}

\subsection{Algorithm: Voro-POMCPOW}
The algorithm for Voro-POMCPOW, which incorporates Voronoi partitioning for continuous observation spaces, is shown below in Algorithm~\ref{alg:voro-pomcpow}.
\begin{algorithm}
\caption{Voro-POMCPOW}
\label{alg:voro-pomcpow}
\begin{algorithmic}[1]
\Procedure{Simulate}{$s, h, d$}
    \If{$d \ge L$}
        \State \Return $0$
    \EndIf

    \State $a \gets \textsc{ActionSelection}(h)$
    \State $(s', z, r) \gets G(s, a)$

    \If{$\lvert \voro(z \mid ha) \rvert \le k_z N(ha)^{\alpha_z}$}
        \State Create a new Voronoi cell $c(z \mid ha)$ with center $z$
        \State $\voro(z \mid ha) \gets \voro(z \mid ha) \cup \{c(z \mid ha)\}$
        \State $total \gets r + \gamma \,\textsc{Rollout}(s', haz, d+1)$
        \State Update value estimates $V(hac)$
    \Else
        \State Find nearest Voronoi cell $c^i \in \voro(z \mid ha)$ to observation $z$
        \State $z \gets \text{Cell center } z^c \text{ of } c^i$
        \State $total \gets r + \gamma \,\textsc{Simulate}(s', haz, d+1)$
        \State Update value estimates $V(hac)$
    \EndIf

    \State $N(h) \gets N(h) + 1$
    \State $N(ha) \gets N(ha) + 1$
    \State $Q(ha) \gets Q(ha) + \dfrac{total - Q(ha)}{N(ha)}$
    \State \Return $total$
\EndProcedure
\end{algorithmic}
\end{algorithm}

\subsection{Parameter Design}
\label{app:parameter-design}
Theorem~\ref{theorem:overall-error} uses two groups of parameters: the corrected-UCB parameters inherited from Theorem~\ref{theorem:pomdp-mcts-ucb} and the progressive-widening parameters \(k_z,\alpha_z\). For the corrected-UCB part, we can fix \(\eta\in[1/2,1)\) without depth dependence, set \(\kappa=\eta(1-\eta)\), and choose a root tail exponent \(\xi^{(0)}>1\). We set
\begin{align}
\xi^{(0)}=\rho,\qquad
\xi^{(\ell+1)}=\frac{\xi^{(\ell)}+1}{\kappa},
\qquad
\alpha^{(\ell+1)}=\xi^{(\ell)}+1=\kappa\xi^{(\ell+1)},
\end{align}
for \(\ell=0,\ldots,L-1\). Hence, at a node \(h\) of depth \(\ell\),
\begin{align}
\frac{\alpha^{(\ell+1)}}{\xi^{(\ell+1)}}=\kappa,
\end{align}
and the corrected-UCB bonus can be implemented as
\begin{align}
B(h,a)
=
c_\ell
\frac{N(h)^{\eta(1-\eta)}}{N(h,a)^{1-\eta}},
\end{align}
where \(c_\ell\) absorbs the conservative concentration constant. In experiments, we use
\begin{align}
c_\ell=c_0V_{\max,\ell},
\qquad
V_{\max,\ell}
=
R_{\max}\frac{1-\gamma^{L-\ell}}{1-\gamma},
\end{align}
with \(V_{\max,\ell}=(L-\ell)R_{\max}\) when \(\gamma=1\), and tune only \(c_0\). For the default choice \(\eta=1/2\), this becomes
\begin{align}
B(h,a)
=
c_\ell
\frac{N(h)^{1/4}}{\sqrt{N(h,a)}}.
\end{align}
Thus the theoretical parameters \(\xi^{(\ell)}\) and \(\alpha^{(\ell)}\) are depth-dependent, but the count exponents of the implemented bonus are not.

For continuous observations, we use progressive widening
\(
m(\tilde h a)\le k_z N(\tilde h a)^{\alpha_z}.
\)
We can set the progressive widening parameters according to the continuous space and the desired partitioning rate. 


\section{Experimental settings}
\label{appendix:experiment}
Experimental settings are as follows.

\subsection{Modified Light-Dark 1D Model}
The Light-Dark 1D environment is a continuous-state POMDP defined as follows:
\begin{itemize}
    \item \textbf{Transitions:} $x' = x + a \cdot \Delta + w$, where $w \sim \mathcal{N}(0, \sigma_{\text{tr}}^2)$ and $\Delta$ is the step size.
    \item \textbf{Observations:} $o \sim \mathcal{N}(x', \sigma_{\text{obs}}(x')^2)$. The noise scales with distance from the light source $x_{\text{light}}$:
    \begin{align}
        \sigma_{\text{obs}}(x) = \sigma_{\text{min}} + \alpha |x - x_{\text{light}}|
    \end{align}
    where $\alpha$ is the observation slope.
    \item \textbf{Reward:} $R(s, a, s') = -|s' - g| - C_{\text{step}}$, where $g$ is the goal location.
\end{itemize}
\textbf{Parameters:} Step size $\Delta=1.0$, $\sigma_{\text{tr}}=0.1$, $x_{\text{light}}=0$, $\sigma_{\text{min}}=0.1$, $\alpha=0.5$, goal $g=0$ with radius $0.2$, $C_{\text{step}}=0.1$, and $\gamma=0.95$.

\subsection{Parameters}
\paragraph{Modified 1D LightDark environment.}
We use a compact, theorem-friendly variant of the 1D LightDark POMDP. The state, observation, and action spaces are
$\mathcal{X}=[-1,1], \ \mathcal{Z}=[-1.5,1.5], \ \mathcal{A}=\{-0.4,0,0.4\}.$
The light source is located at $x_{\mathrm{light}}=0$, and the goal is at $x_{\mathrm{goal}}=0.8$. The transition model is
\begin{align}
x_{t+1}
&=
\operatorname{Gate}_{[-1,1]}\bigl(x_t+a_t+w_t\bigr),
w_t \sim \mathrm{TruncNormal}(0,\sigma_T^2;[-w_{\max},w_{\max}]),
\end{align}
where $\operatorname{Gate}_{[-1,1]}(y)=\min\{1,\max\{-1,y\}\}$, $\sigma_T=0.02$, and $w_{\max}=0.06$. The observation noise is state dependent: $\sigma_o(x)=\sigma_{\min}+(\sigma_{\max}-\sigma_{\min})|x-x_{\mathrm{light}}|,$
with $\sigma_{\min}=0.05$ and $\sigma_{\max}=0.35$. To ensure a nonzero observation-density floor, observations are sampled from the mixture
\begin{align}
z_t
&\sim
(1-k_{\mathrm{floor}})
\mathrm{TruncNormal}
\bigl(x_t,\sigma_o(x_t)^2;[-1.5,1.5]\bigr)
+
k_{\mathrm{floor}}
\mathrm{Unif}([-1.5,1.5]),
\end{align}
where $k_{\mathrm{floor}}=0.05$.
The reward is bounded in $[0,1]$:
\begin{align}
r(x,a)
&=
1-
\min\left\{
1,
\frac{|x-x_{\mathrm{goal}}|}{2}
+
\lambda \frac{|a|}{0.4}
+
\tau\,\mathrm{r}_{a}(a)
\right\},
\end{align}
where $\lambda=0.05$, $\tau=10^{-4}$, and $\mathrm{r}_{a}(-0.4)=0,\ \mathrm{r}_{a}(0)=1,\ \mathrm{r}_{a}(0.4)=2.$
The initial belief is $b_0=\mathrm{TruncNormal}(-0.6,0.15^2;[-1,1]).$

\paragraph{Experiment setting.}
We use horizon \(L=3\), discount \(\gamma=0.95\), and \(n=5000\) simulations
per search. Observation progressive widening uses \(k_z=8\) and
\(\alpha_z=0.5\). The depth-scaled corrected-UCB bonus uses \(\eta=1/2\),
\(c_\ell=c_0V_{\max,\ell}\), \(c_0=1\), and \(R_{\max}=1\). We set
\(\xi^{(0)}=2\) and \(\beta^{(0)}=2\), which gives
\(\xi^{(0:3)}=(2,12,52,212)\) and \(\alpha^{(1:3)}=(3,13,53)\); the
leaf-to-root recursion checks are verified before reporting the certificate.
For confidence \(1-\delta\), we use \(\delta_1=0.75\delta\) and
\(\delta_2=0.25\delta\).
The partition term uses \(C=20\),
\(k_Z=1\), and radius cap \(1\). Results are averaged over 100 independent
runs with seeds $20260504-20260603$.

The experiment is implemented in Julia and running on a machine with an i7-12700 CPU and 32GB RAM. 

\section{Proof of Theorem~\ref{theorem:pomdp-mcts-ucb}}
\label{appendix:proof-pomdp-mcts-ucb}
The parameters are updated recursively as follows:
\begin{enumerate}
    \item \textbf{Leaf Initialization:} Let $\xi^{(L)}$ be a sufficiently large constant dependent on the bounded reward support at the leaves.
    \item \textbf{Recursive Step:} For each depth $h = L-1, \dots, 0$, choose the exploration parameter $\alpha^{(h+1)}$ satisfying:
    \begin{align}
        \xi^{(h+1)} \eta (1-\eta) \le \alpha^{(h+1)} < \xi^{(h+1)} (1-\eta),
    \end{align}
    and define the resulting concentration exponent for depth $h$ as:
    \begin{align}
        \xi^{(h)} = \alpha^{(h+1)} - 1.
    \end{align}
    \item \textbf{Condition:} The leaf parameter $\xi^{(L)}$ is chosen large enough such that $\alpha^{(1)} > 2$.
\end{enumerate}

The proof of Theorem~\ref{theorem:pomdp-mcts-ucb} largely follows the structure of \citet{Shah22or} but extending to the discrete POMDP setting. The main difference lies in the \emph{observation mixing} step that handles the branching due to the observation and transition model in POMDPs. We provide a proof sketch below.

\begin{proof}[Proof Sketch of Theorem \ref{theorem:pomdp-mcts-ucb}]
The proof proceeds by backward induction on the depth of the tree, from the leaf level ($h=L$) to the root ($h=0$). We define a random return process $\{G_t\}$ to satisfy the polynomial concentration property $\mathcal{P}(\beta, \xi, \eta)$ (same as~\citep{Shah22or}) around a mean $V^*$ if for all $n \ge 1, z \ge 1$:
\begin{align}
    \mathbb{P}\left( n(\hat{G}_n - V^*) \ge n^{\eta} z \right) \le \beta z^{-\xi},
\end{align}
and similarly for the lower tail.

\paragraph{1. Base Case (Depth $h=L$)}
At the leaf nodes, the value estimate is derived from the immediate bounded reward. It satisfies Hoeffding's inequality, which implies polynomial concentration $\mathcal{P}(\beta^{(L)}, \xi^{(L)}, \eta)$ for any $\eta \in [1/2, 1)$ and sufficiently large $\xi^{(L)}$.

\paragraph{2. Inductive Step (Depth $h+1 \to h$)}
Assume that for all belief nodes at depth $h+1$, the MCTS value estimates satisfy $\mathcal{P}(\beta^{(h+1)}, \xi^{(h+1)}, \eta)$. We show that the value estimate at a parent node $b$ at depth $h$ satisfies $\mathcal{P}(\beta^{(h)}, \xi^{(h)}, \eta')$. This step decomposes into two parts: Observation Mixing and Action Selection.

\paragraph{Part A: Observation Mixing}
Consider a fixed action $a$ at belief $b$. Executing $a$ yields a stochastic observation $z \in \mathcal{Z}$ with probability $p_z \ge \phi_Z$ (Assumption~4.2). The empirical return $\hat{Q}_{a,n}$ after $n$ visits to $(b,a)$ is a weighted sum of the children's value estimates. Let $N_z(n)$ be the visit count of observation $z$. We decompose the error:
\begin{align}
    n(\hat{Q}_{a,n} - Q_a^*) = \underbrace{\sum_{z \in \mathcal{Z}} N_z(n) (\hat{V}_{z, N_z(n)} - V_z^*)}_{\text{Recursive Estimation Error}} + \underbrace{\sum_{z \in \mathcal{Z}} (N_z(n) - n p_z) V_z^*}_{\text{Stochastic Observation Noise}}.
\end{align}
\begin{itemize}
    \item The \textit{Recursive Estimation Error} is a sum of terms that, conditioned on $N_z(n)$, satisfy the inductive hypothesis $\mathcal{P}(\beta^{(h+1)}, \xi^{(h+1)}, \eta^{(h+1)})$. Since $\eta^{(h+1)} \ge 1/2$, the polynomial concentration is preserved under this weighted summation (see Lemma A.1 in \citep{Shah22or}).
    \item The \textit{Stochastic Observation Noise} arises from the deviation of counts $N_z(n)$ from their expected values $n p_z$. These counts follow a multinomial distribution and thus exhibit Sub-Gaussian concentration. Since Sub-Gaussian tails decay faster than polynomial tails, this noise is absorbed into the polynomial bound without degrading the exponent $\xi^{(h+1)}$.
\end{itemize}
Thus, the mixed return process for action $a$ satisfies $\mathcal{P}({\beta}^{(h+1)}, \xi^{(h+1)}, \eta^{(h+1)})$.

\paragraph{Part B: Action Selection (Bandit)}
We treat the node $b$ as a non-stationary bandit problem where each arm $a$ satisfies $\mathcal{P}({\beta}^{(h+1)}, \xi^{(h+1)}, \eta^{(h+1)})$. We apply the Corrected-UCB rule with parameter $\alpha^{(h+1)}$.
By Theorem 3 of \citet{Shah22or}, provided that $\xi^{(h+1)}\eta^{(h+1)}(1-\eta^{(h+1)}) \le \alpha^{(h+1)} < \xi^{(h+1)}(1-\eta^{(h+1)})$, the aggregated value estimate at node $b$ satisfies the polynomial concentration $\mathcal{P}(\beta^{(h)}, \xi^{(h)}, \eta^{(h)})$ with:
\begin{align}
    \xi^{(h)} = \alpha^{(h+1)} - 1, \quad \text{and} \quad \eta^{(h)} = \frac{\alpha^{(h+1)}}{\xi^{(h+1)}(1-\eta^{(h+1)})}.
\end{align}
This confirms the recursive parameter updates defined in the theorem.

\paragraph{3. Conclusion}
By backward induction, the root node $b_0$ satisfies $\mathcal{P}(\beta^{(0)}, \xi^{(0)}, \eta^{(0)})$. Specifically,
\begin{align}
    \mathbb{P}\left( n|\hat{V}_n(b_0) - V^*(b_0)| \ge n^{\eta^{(0)}} z \right) \le \beta^{(0)} z^{-\xi^{(0)}},
\end{align}
which implies $|\hat{V}_n(b_0) - V^*(b_0)| = O_p(n^{-(1-\eta^{(0)})})$.
\end{proof}

\section{Proof of Theorem 2}
\label{app:proof_theorem_2}


\begin{proof}
Fix a planning horizon $L$. Let $\tilde{h}a$ be a possible history-action node. Let $m(\tilde{h}a) = |\mathcal{C}(\tilde{h}a)|$ be the number of cells at that node. We denote the set of cell centers as $\{z^c_1, \dots, z^c_{m}\}$. For any observation $z \in \mathcal{Z}$, let $q_{\tilde{h}a}(z) = \arg\min_{z^c \in \mathcal{C}(\tilde{h}a)} D(z, z^c)$ be the nearest-center quantizer.
By the covering radius definition $\tilde{\epsilon}_{\tilde{h}a}$ in Assumption~\ref{ass:covering_dvp}, we have $D(z, q_{\tilde{h}a}(z)) \le \tilde{\epsilon}_{\tilde{h}a}$ for all $z$.

\paragraph{Step 1: One-Step Error ($Q^m$).}
To analyze the error propagation, we introduce an intermediate Q-function, $Q^m(\tilde{h}, a)$, which represents the value of taking action $a$ at history $\tilde{h}$ using the \textit{quantized} observation for the immediate step, but assuming the \textit{optimal} value function $V^*$ for all future steps.

Let the true Q-value be:
\begin{align}
    Q^*(\tilde{h}, a) = r(b, a) + \gamma \mathbb{E}_{z \sim \mathbb{P}_Z} [V^*(\psi(b, a, z))].
\end{align}
We define the modified Q-value $Q^m$ as:
\begin{align}
    Q^m(\tilde{h}, a) \triangleq r(b, a) + \gamma \mathbb{E}_{z \sim \mathbb{P}_Z} [V^*(\psi(b, a, q_{\tilde{h}a}(z)))].
\end{align}
Note that $Q^m$ uses $V^*$ in the expectation isolating the local error caused by center represented belief update.

We bound the difference $|Q^*(\tilde{h}, a) - Q^m(\tilde{h}, a)|$.
By Assumption~\ref{ass:holder-obs} and ~\ref{ass:holder-V-function}, we have:
\begin{align}
    |V^*(\psi(b, a, z)) - V^*(\psi(b, a, q_{\tilde{h}a}(z)))| \le L_V L_\psi^\beta D(z, q_{\tilde{h}a}(z))^{\alpha\beta} \le L_V L_\psi^\beta (\tilde{\epsilon}_{\tilde{h}a})^{\alpha\beta}.
\end{align}
Integrating over $z$:
\begin{align}
    |Q^*(\tilde{h}, a) - Q^m(\tilde{h}, a)|
    &\le \gamma \int_{z \in \mathcal{Z}} \mathbb{P}_Z(z|\tilde{h},a) |V^*(\psi(b, a, z)) - V^*(\psi(b, a, q_{\tilde{h}a}(z)))| dz \\
    &\le \gamma L_V L_\psi^\beta (\tilde{\epsilon}_{\tilde{h}a})^{\alpha\beta}. \label{eq:local_q_error}
\end{align}

\paragraph{Step 2: Recursive Error.}
We now relate the optimal value function $V^*$ to the partitioned optimal value function $V^{p*}$.
Recall that $V^*(\tilde{h}) = \max_a Q^*(\tilde{h}, a)$ and $V^{p*}(\tilde{h}) = \max_a Q^{p*}(\tilde{h}, a)$.
Using the inequality $|\max f - \max g| \le \max |f - g|$:
\begin{align}
    |V^*(\tilde{h}) - V^{p*}(\tilde{h})| \le \max_{a} |Q^*(\tilde{h}, a) - Q^{p*}(\tilde{h}, a)|.
\end{align}
We apply the triangle inequality using our intermediate term $Q^m$:
\begin{align}
    |Q^*(\tilde{h}, a) - Q^{p*}(\tilde{h}, a)| \le \underbrace{|Q^*(\tilde{h}, a) - Q^m(\tilde{h}, a)|}_{\text{Local Quantization Error}} + \underbrace{|Q^m(\tilde{h}, a) - Q^{p*}(\tilde{h}, a)|}_{\text{Iterative Propagation Error}}.
\end{align}
The first term is bounded by Eq. (\ref{eq:local_q_error}). The second term captures the difference in future values:
\begin{align}
    |Q^m(\tilde{h}, a) - Q^{p*}(\tilde{h}, a)|
    &= \left| \gamma \mathbb{E}_{z} [V^*(\psi(b, a, q(z)))] - \gamma \mathbb{E}_{z} [V^{p*}(\psi(b, a, q(z)))] \right| \\
    &\le \gamma \mathbb{E}_{z} \left| V^*(\psi(\cdot)) - V^{p*}(\psi(\cdot)) \right| 
    \le \gamma \max_{b'} |V^*(b') - V^{p*}(b')|.
\end{align}
Thus, the total error at depth $t$ satisfies the recurrence:
\begin{align}
    \epsilon_t \le \gamma L_V L_\psi^\beta \max (\tilde{\epsilon})^{\alpha\beta} + \gamma \epsilon_{t+1}.
\end{align}

\paragraph{Step 3: Finite-Horizon Accumulation.}
Iterate recursively over the horizon $L$, we can get:
\begin{align}
    |V^*(b_0) - V^{p*}(b_0)| \le \sum_{t=0}^{L-1} \gamma^{t+1} L_V L_\psi^\beta \max_{\tilde{h}a \in \mathcal{H}_L} (\tilde{\epsilon}_{\tilde{h}a})^{\alpha\beta} = \frac{\gamma(1-\gamma^L)}{1-\gamma} L_V L_\psi^\beta \left( \max_{\tilde{h}a} \tilde{\epsilon}_{\tilde{h}a} \right)^{\alpha\beta}.
\end{align}

\paragraph{Step 4: Probabilistic Bound.}
Finally, we bound $\tilde{\epsilon}_{\tilde{h}a}$. By Assumptions 5.3 and 5.4, the covering radius with $m$ samples satisfies $\mathbb{P}(\tilde{\epsilon}_{\tilde{h}a} > 2r) \le C_{cov} r^{-k_Z} \exp(-m c_\nu r^{k_Z})$.
Solving for $r$ such that the probability is bounded by $\delta/|\mathcal{H}_L|$ (union bound over the tree) yields the high-probability bound:
\begin{align}
    \tilde{\epsilon}_{\tilde{h}a} \le C \left( \frac{\log(m(\tilde{h}a) |\mathcal{H}_L| / \delta)}{m(\tilde{h}a)} \right)^{1/k_Z}.
\end{align}
Substituting this into the accumulated error completes the proof.
\end{proof}

\section{Proof of Theorem~\ref{theorem:overall-error}}
\label{proof:overall-error}
\begin{proof}[Sketch of Proof for Theorem 3]
The proof relies on decomposing the total error into two distinct components: the \textit{approximation error} (induced by discretizing the continuous space) and the \textit{estimation error} (induced by the finite sampling in MCTS).

Let $V^*(b_0)$ denote the optimal value of the original continuous POMDP, and $\hat{V}_n^p(b_0)$ denote the MCTS value estimate at the root after $n$ simulations using the adaptive partitioning. We introduce an intermediate term $V^{p*}_n(b_0)$, which represents the optimal value of the \textit{specific} discrete POMDP instance induced by the Voronoi partition $\mathcal{P}_n$ existing at iteration $n$.

Using the triangle inequality, we decompose the total error as:
\begin{align}
    |V^*(b_0) - \hat{V}_n^p(b_0)| \le \underbrace{|V^*(b_0) - V^{p*}_n(b_0)|}_{\text{Approximation Error}} + \underbrace{|V^{p*}_n(b_0) - \hat{V}_n^p(b_0)|}_{\text{Estimation Error}}
\end{align}

\paragraph{Bounding the Approximation Error (Term 1):}
This term can be bounded directly from Theorem~\ref{theorem:partition-loss}.

\paragraph{Bounding the Estimation Error (Term 2).}
Here, we use the Theorem~\ref{theorem:pomdp-mcts-ucb}'s result to bound the term. However, a minor modification will be made. For the Proof Part A of theorem~\ref{theorem:pomdp-mcts-ucb}, we need to make a modification, the stochastic observation noise is now non-stationary.Modification is provided at the end of this section.

\paragraph{Union Bound:}
Theorem 3 is obtained by taking a union bound over the events where both the partitioning requirements (Theorem 2) and the MCTS concentration (Theorem 1 with modification) hold. By setting error probabilities $\delta_1 = \delta_2 = \delta/2$, we obtain the final finite-time bound for POMDPs with the continuous observation space.

\subsection*{Modification to Proof Part A's Observation Noise of Theorem~\ref{theorem:pomdp-mcts-ucb}}
This is to make modification to the proof of Theorem~\ref{theorem:pomdp-mcts-ucb} to handle the non-stationary observation noise. If the partitioning is adaptive or in a progressive widening way, we need the following modification to the proof of Theorem~\ref{theorem:pomdp-mcts-ucb} with a martingale-based analysis to handle the non-stationarity of the discretized observation process. The key idea is to show that the observation-mixing noise term remains polynomially concentrated even when the discretized observations are non-stationary, which allows us to maintain the same concentration properties required for the UCB analysis.
 We only need to modify the Observation Noise bound in the proof  Part A of Theorem~\ref{theorem:pomdp-mcts-ucb}.

Fix a history--action node $\tilde h a$ in the tree induced by the current partition.
Index by $t=1,2,\dots$ the successive visits to $\tilde h a$ during the simulations, and let $\mathcal F_{t-1}$ be the $\sigma$-field generated by all randomness up to just before the $t$-th such visit.
Due to the adaptive (time-varying) quantizer/partition used to discretize observations, the discrete child symbol $\tilde z_t\in\mathcal Z_n(\tilde h a)$ may be \emph{non-stationary}.
Define the conditional (time-varying) categorical probabilities
\begin{align}
p_{u,t}\triangleq \mathbb P(\tilde z_t=u\mid \mathcal F_{t-1}),\qquad u\in\mathcal Z_n(\tilde h a).
\end{align}
Let $N_u(m)\triangleq \sum_{t=1}^m \mathbf 1\{\tilde z_t=u\}$ be the number of times child $u$ is selected among the first $m$ visits. Then
\begin{align}
N_u(m)-\sum_{t=1}^m p_{u,t}
=\sum_{t=1}^m\Bigl(\mathbf 1\{\tilde z_t=u\}-p_{u,t}\Bigr)
=:M_u(m),
\end{align}
where $(M_u(m))_{m\ge 1}$ is a martingale with respect to $(\mathcal F_m)$ and has bounded increments $|M_u(t)-M_u(t-1)|\le 1$.

Consider the observation-mixing noise term
\begin{align}
S(m)\triangleq \sum_{u\in\mathcal Z_n(\tilde h a)} V_u^\ast\Bigl(N_u(m)-\sum_{t=1}^m p_{u,t}\Bigr)
=\sum_{t=1}^m\Bigl(V^\ast_{\tilde z_t}-\mathbb E[V^\ast_{\tilde z_t}\mid \mathcal F_{t-1}]\Bigr),
\end{align}
where $V_u^\ast$ denotes the optimal value at the child corresponding to $u$.
Hence $(S(m))_{m\ge 1}$ is also a martingale. Since rewards are bounded and the horizon is finite, there exists $V_{\max}<\infty$ such that $|V_u^\ast|\le V_{\max}$ for all $u$, and thus
\begin{align}
|S(t)-S(t-1)|
=\bigl|V^\ast_{\tilde z_t}-\mathbb E[V^\ast_{\tilde z_t}\mid \mathcal F_{t-1}]\bigr|
\le 2V_{\max}.
\end{align}
By Azuma--Hoeffding, for any $x>0$,
\begin{align}
\mathbb P\bigl(|S(m)|\ge x\bigr)\le 2\exp\!\left(-\frac{x^2}{8mV_{\max}^2}\right).
\end{align}
This bound does not require stationarity of $(\tilde z_t)$. Moreover, exponential tails imply the polynomial tail used in Theorem~\ref{theorem:pomdp-mcts-ucb}: for any $\xi>0$, let
\begin{align}
\beta_\xi \triangleq \sup_{z\ge 1} 2\,z^\xi \exp\!\left(-\frac{z^2}{8V_{\max}^2}\right)<\infty.
\end{align}
Then for any $\eta\ge \tfrac12$ and all $m\ge 1$, $z\ge 1$,
\begin{align}
\mathbb P\bigl(|S(m)|\ge m^\eta z\bigr)
\le 2\exp\!\left(-\frac{m^{2\eta-1}z^2}{8V_{\max}^2}\right)
\le 2\exp\!\left(-\frac{z^2}{8V_{\max}^2}\right)
\le \beta_\xi\, z^{-\xi}.
\end{align}
Therefore, the observation-mixing noise remains polynomially concentrated (with the same exponent $\eta$) even under non-stationary discretized observations, and the rest of Proof Part A (and hence Theorem~\ref{theorem:pomdp-mcts-ucb}) proceeds unchanged.

\end{proof}

\end{document}